\pdfoutput=1
\documentclass{article}

     \usepackage[preprint]{neurips_2025}

\usepackage[utf8]{inputenc} 
\usepackage[T1]{fontenc}    
\usepackage{hyperref}       
\usepackage{url}            
\usepackage{booktabs}       
\usepackage{amsfonts}       
\usepackage{nicefrac}       
\usepackage{microtype}      
\usepackage{xcolor}         

\usepackage{framed}
\usepackage{balance}
\usepackage{epsfig}
\usepackage{color}
\usepackage{subfigure}
\usepackage{multirow,tabularx}
\usepackage{mathtools}
\usepackage{graphicx}
\usepackage{multirow}
\usepackage{bm}
\usepackage{csquotes}
\usepackage{array}
\usepackage[yyyymmdd,hhmmss]{datetime}
\usepackage{xcolor,colortbl}

\usepackage{pifont}
\usepackage[algo2e,titlenumbered,ruled]{algorithm2e} 
\usepackage{lipsum,environ,amsmath}
\usepackage{slashbox,booktabs}

\usepackage{pdflscape}

\newcounter{Lcount}
\newcommand{\numsquishlist}{
   \begin{list}{\arabic{Lcount}. }
    { \usecounter{Lcount}
 \setlength{\itemsep}{-.1ex}      \setlength{\parsep}{0ex}
      \setlength{\topsep}{0ex}       \setlength{\partopsep}{0ex}
      \setlength{\leftmargin}{1em} \setlength{\labelwidth}{1em}
      \setlength{\labelsep}{0.1em} } }
\newcommand{\numsquishend}{\end{list}}

\newcommand{\squishlist}{
   \begin{list}{$\bullet$}
    { \setlength{\itemsep}{-.1ex}      \setlength{\parsep}{0ex}
      \setlength{\topsep}{0ex}       \setlength{\partopsep}{0ex}
      \setlength{\leftmargin}{.8em} \setlength{\labelwidth}{1em}
      \setlength{\labelsep}{0.5em} } }
\newcommand{\squishend}{\end{list}}

\newcommand{\VLSSIP}{{\sc Variable-length Similar Subsequences Inference Problem}\xspace}%

\SetCommentSty{mycommfont}

\newenvironment{problem}[1][htb]
  {
   \begin{algorithm2e}[#1]%
  }{\end{algorithm2e}}

\AtBeginDocument{%
  \providecommand\BibTeX{{%
    \normalfont B\kern-0.5em{\scshape i\kern-0.25em b}\kern-0.8em\TeX}}}

\title{Inferring the Most Similar Variable-length Subsequences between Multidimensional Time Series}

\author{%
  Thanadej~Rattanakornphan\thanks{First author}\\
  Department of Computer Engineering\\
  Kasetsart University\\
  Bangkok, Thailand \\
  \texttt{ra.thanadej@gmail.com} \\
   \And
  Piyanon Charoenpoonpanich \\
  Independent \\
  Bangkok, Thailand \\
  \texttt{piyanon.charoenpoonpanich@gmail.com} \\
  \AND
  Chainarong Amornbunchornvej \thanks{Corresponding author}\\
  National Electronics and Computer Technology Center \\
  Pathumthani, Thailand \\
  \texttt{chainarong.amo@nectec.or.th} \\
}

\begin{document}

\maketitle

\begin{abstract}
  Finding the most similar subsequences between two multidimensional time series has many applications: e.g. capturing dependency in stock market or discovering coordinated movement of baboons. Considering one pattern occurring in one time series, we might be wondering whether the same pattern occurs in another time series with some distortion that might have a different length. Nevertheless, to the best of our knowledge, there is no efficient framework that deals with this problem yet. In this work, we propose an algorithm that provides the exact solution of finding the most similar multidimensional subsequences between time series where there is a difference in length both between time series and between subsequences. The algorithm is built based on theoretical guarantee of correctness and efficiency. The result in simulation datasets illustrated that our approach not just only provided correct solution, but it also utilized running time only quarter of time compared against the baseline approaches. In real-world datasets, it extracted the most similar subsequences even faster (up to 20 times faster against baseline methods) and provided insights regarding the situation in stock market and following relations of multidimensional time series of baboon movement. Our approach can be used for any time series. The code and datasets of this work are provided for the public use.
\end{abstract}



\section{Introduction}

Every second, streams of data have been generated enormously from both human (e.g. online social networks, content creators) and machines (e.g. data logs of systems, GPS coordinates of devices).  With these mountains of data stream, there are countless patterns and insight that might be inferred and utilized from the data using methods from knowledge discovery from data (KDD). \textit{Time Series Analysis} is one of the KDD fields that focuses on streaming of data, which deals with modeling and discovering patterns from streams of data or time series. 

\begin{framed}
\noindent {\VLSSIP:} {Considering one pattern occurs in one time series, we might be wondering whether the same pattern occurs in other time series with some distortion.  {\bf Given two time series, the goal is to find the most similar subsequences from each time series that might have different lengths. }}
\end{framed}

One of the problems that is important in time series analysis is the problem of searching similar patterns or subsequences within large sets of time series~\cite{10.1145/1066157.1066235,10.1145/2339530.2339576}. To find similar patterns, the first step is to measure a similarity between time series.  Dynamic Time Warping (DTW)~\cite{sakoe1978dynamic} is one of widely-used distance measures~\cite{giorgino2009computing} since it can find generic and distorted patterns between time series~\cite{10.1145/2339530.2339576}. However, to the best of our knowledge, there is no existing method developed to efficiently find the most similar subsequences within two time series s.t. the subsequence in one time series might have a different length compared to the subsequence of another time series. 

In this work, we proposed a generalization measure of DTW s.t. the proposed measure is able to find most similar subsequences within DTW scheme. In the case that the lengths of subsequences are the same as the original time series, the proposed method works the same as typical DTW. Additionally, in the case that the lengths of subsequences are shorter than the original time series, our proposed method provides exact solution of most similar subsequences by using more efficient computational resources than using DTW to search for the solution. The proposed method can be used for any kind of multidimensional time series.

\section{Related work}
Searching for similarity patterns in time series have been studies in literature for many years~\cite{10.1145/1066157.1066235}.  There are several methods that deal with different types of patterns in time series such as motifs~\cite{alaee2020matrix,imamura2024efficient}, discords~\cite{zhu2016matrix}, clustering~\cite{holder2024review} etc.
The typical way of finding similarity patterns is to use some distance/correlation functions of time series such as cross-correlation~\cite{kjaergaard2013time}, Levenshtein distance~\cite{10.1145/375360.375365}, Longest Common Subsequence (LCSS)~\cite{soleimani2020dlcss}, Fréchet distance~\cite{driemel2016clustering} etc.
Among these measures, DTW~\cite{sakoe1978dynamic} is one of the widely used approach to measure a distance between two time series since it can handle distortion of similar patterns between time series~\cite{10.1145/2339530.2339576}. Several versions of DTW have been developed. The classic one is to use window constraints to limit search space (e.g. Sakoe-Chiba band)~\cite{sakoe1978dynamic,geler2022elastic}. The works in \cite{keogh2005exact,ratanamahatana2005three,vlachos2003indexing} used a LB\_Keogh lower bounding to enhance DTW. The work in~\cite{wang2016graphical} enhanced DTW to infer multiple aliments of time series using the network flow. The work in~\cite{alaee2020matrix} can infer one-dimensional most similar subsequences using Matrix profile inference algorithm and lower/upper bounds but the window of subsequences must be similar. 

Nevertheless, there is no direct version of DTW that can find the most similar multidimensional subsequences between time series where there is a difference in length both between time series and between subsequences. The obvious solution is using DTW in the brute-force way, which can provide a solution with an expensive cost.

Hence, in this work, we propose a generalized version of DTW that can handle the problem of finding the most similar subsequences between time series that have different length that can perform efficiently. The proposed algorithm provides the following new properties. 

\squishlist
\item {\bf Inferring arbitrary-length similar subsequences:} our approach can infer a pair of the most similar multidimensional subsequences that can have different lengths; 
\item {\bf Ranking top-$k$ similar subsequences:} our approach ranks top-$k$ of most similar subsequences.  
\squishend

Our proposed approach can be used for any kind of multidimensional time series.

\newtheorem{definition}{Definition}
\newtheorem{theorem}{Theorem}[section]
\newtheorem{lemma}[theorem]{Lemma}
\newtheorem{proposition}[theorem]{Proposition}
\newtheorem{corollary}[theorem]{Corollary}

\newenvironment{proof}[1][Proof]{\begin{trivlist}
\item[\hskip \labelsep {\bfseries #1}]}{\end{trivlist}}
\newenvironment{example}[1][Example]{\begin{trivlist}
\item[\hskip \labelsep {\bfseries #1}]}{\end{trivlist}}
\newenvironment{remark}[1][Remark]{\begin{trivlist}
\item[\hskip \labelsep {\bfseries #1}]}{\end{trivlist}}

\newcommand{\qed}{\nobreak \ifvmode \relax \else
      \ifdim\lastskip<1.5em \hskip-\lastskip
      \hskip1.5em plus0em minus0.5em \fi \nobreak
      \vrule height0.75em width0.5em depth0.25em\fi}

\section{Problem Statement}
\label{sec:probstat}
First, we define the necessary definitions.    

\begin{definition}[Time series and its subsequence]
\label{def:TS}
Given a feature space $\mathcal{F}$ (e.g. $\mathbb{R}^n$), time series $U$ is a sequence $(u_1,\dots,u_l)$ of length $l$ whose element $u_i\in \mathcal{F}$ is a value at time step $i \in [1..l]$ where $[n..m] = \{x \in \mathbb{Z} \mid n \leq x \leq m\}$. Subsequence $U[a,b]$ is a sequence $(u_a,\dots,u_b)$ whose element $u_{j \in [a..b]}$ is an element in $U$. Subsequence $U_{k,\omega_U} = U[k,k + \omega_U - 1]$ where $\omega_U \in \mathbb{Z}^+$ is a time window and $k$ is a starting index. Note that subsequence $U[a,b]$ is also a time series and $U[1,l]=U$.
\end{definition}

\begin{definition}[Warping path]
\label{def:WP}
Warping path P between subsequence $U[a,b]=(u_a,\dots,u_b)$ and subsequence $W[c,d]=(w_c,\dots,w_d)$ is a sequence $(p_1,p_2,\dots,p_n)$ whose element $p_{k\in [1..n]}=(i_k,j_k) \in [a..b] \times [c..d]$. $P$ satisfies these conditions:

(1) Boundary condition: $p_1=(a,c)$ and $p_n=(b,d)$

(2) Continuity and monotonicity condition: $p_{k'+1}-p_{k'} \in \{(0,1),(1,0),(1,1)\}$ where $k' \in [1..n-1]$
\end{definition}

\begin{definition}[Distance between subsequences of time series]
\label{def:DT}
Distance between subsequence $U[a,b]=(u_a,\dots,u_b)$ and subsequence $W[c,d]=(w_c,\dots,w_d)$ is given by a warping path $P$ between them. If $P = (p_1,p_2,\dots,p_n) = ((i_1,j_1),(i_2,j_2),\dots,(i_n,j_n))$, then the distance between $U[a,b]$ and $W[c,d]$ given $P$, denoted by $DistSub(U[a,b],W[c,d],P)$, is $\sum_{k=1}^{n} dist(u_{i_k},w_{j_k})$ where $dist:\mathcal{F} \times \mathcal{F} \rightarrow \mathbb{R}_{\geq 0}$ is a distance function.
\end{definition}

\begin{definition}[Dynamic time warping distance between subsequences of time series]
\label{def:DTW}
Dynamic time warping distance between subsequence $U[a,b]=(u_a,\dots,u_b)$ and subsequence $W[c,d]=(w_c,\dots,w_d)$, denoted by $DTW(U[a,b],W[c,d])$, is the minimum distance between $U[a,b]$ and $W[c,d]$. In other words, $DTW(U[a,b],W[c,d]) = \min_P DistSub(U[a,b],W[c,d],P)$
\end{definition}

\begin{definition}[Distance matrix of time series]
\label{def:DMT}
    Given time series $U =(u_1,\dots,u_n)$ and time series $W = (w_1,\dots,w_m)$, distance matrix $M_{U, W}$, is a $n \times  m$ matrix where $M[i,j] = dist( u_i,w_j)$
\end{definition}

\begin{definition}[Lower bound of DTW]
\label{def:DTWL}
    Given subsequence $U_{i,\omega_U}$, $W_{j,\omega_W}$ where $\omega_U \geq \omega_W$, and a distance matrix $M_{U, W}$, lower bound of DTW between $U_{i,\omega_U}$ and $W_{j,\omega_W}$, denoted by $DTW_L(U_{i,\omega_U}, W_{j,\omega_W})$, $ = \sum_{k=i}^{i+\omega_U-1} \min \{M[k ; j,\dots,j+\omega_W-1]\}$ where $X[a,\dots,b ; c,\dots,d]$ is a submatrix formed by taking row $a$ to $b$ and column $c$ to $d$ from matrix $X$.
\end{definition}

\begin{definition}[Upper bound of DTW]
\label{def:DTWU}
    Given subsequence $U_{i,\omega_U}$, $W_{j,\omega_W}$ where $\omega_U \geq \omega_W$, and a distance matrix $M_{U, W}$, upper bound of DTW between $U_{i,\omega_U}$ and $W_{j,\omega_W}$, denoted by $DTW_{up}(U_{i,\omega_U}, W_{j,\omega_W})$, $= \sum_{k=0}^{\omega_W-2} M[i+k,j+k] + \sum_{k=\omega_W-1}^{\omega_U-1} M[i+k,j+\omega_W-1]$
\end{definition}

\begin{definition}[Domain of interest]
\label{def:DI}
    Domain of interest $D_{n,m} = (U,W,\omega_U,\omega_W)$ is a tuple of time series $U=(u_1,\dots,u_n), W=(w_1,\dots,w_m)$ and time window $\omega_U,\omega_W$.
\end{definition}

\begin{definition}[Dynamic time warping matrix]
\label{def:DTWM}
    Given domain of interest $D_{n,m} = (U,W,\omega_U,\omega_W) = D$, dynamic time warping matrix $DTWM_D$, lower bound of DTW matrix $DTWML_D$, and upper bound of DTW matrix $DTWMU_D$ are $(n-\omega_U+1)\times (m-\omega_W+1)$ matrices where 
    
    $DTWM_D[i,j] = DTW(U_{i,\omega_U}, W_{j,\omega_W})$ 
    
    $DTWML_D[i,j] = DTW_L(U_{i,\omega_U}, W_{j,\omega_W})$ 

    $DTWMU_D[i,j] = DTW_{up}(U_{i,\omega_U}, W_{j,\omega_W})$ 
\end{definition}

Now, we are ready to formalize the problem.

\begin{problem}
    \SetKwInOut{Input}{Input}
    \SetKwInOut{Output}{Output}
    \Input{Time series $U$ and $W$, as well as time windows $\omega_U$ and $\omega_W$.}
    \Output{Subsequence $U_{a,\omega_U}$ of length $\omega_U$ and subsequence $W_{b,\omega_W}$ of length $\omega_W$ which are most similar to each other ($a,b$ s.t. $\forall i,j, DTWM_D[a,b] \leq DTWM_D[i,j] $).}
    \caption{\VLSSIP}
    \label{prob:VLSSIP}
\end{problem}

Next, we provide useful propositions and lemma that we use later. The details of the proofs are provided in the Appendix Section~\ref{sec:apxProof}.

\begin{proposition}
Given subsequence $U_{i,\omega_U}$, $W_{j,\omega_W}$ where $\omega_U \geq \omega_W$, and a distance matrix $M_{U, W} = M$, the following inequality holds:

$DTW_L(U_{i,\omega_U}, W_{j,\omega_W}) \leq DTW(U_{i,\omega_U}, W_{j,\omega_W})$

\label{prop:lowerbound}
\end{proposition}


\begin{lemma}
Given subsequence $U_{i,\omega_U}$, $W_{j,\omega_W}$ where $\omega_U \geq \omega_W$, if sequence $P = (p_n)_{n \in [0..(\omega_U-1)]}$ where $p_n = (i+n,j+n)$ if $n < \omega_W-1$ and $(i+n,j+\omega_W-1)$ if $n \geq \omega_W-1$, then $P$ is a warping path between $U_{i,\omega_U}$ and $W_{j,\omega_W}$
\label{lemma:validpath}
\end{lemma}



\begin{proposition}
Given subsequence $U_{i,\omega_U}$, $W_{j,\omega_W}$ where $\omega_U \geq \omega_W$, and a distance matrix $M_{U, W} = M$, the following inequality holds:

$DTW(U_{i,\omega_U}, W_{j,\omega_W}) \leq DTW_{up}(U_{i,\omega_U}, W_{j,\omega_W})$

\label{prop:upperbound}
\end{proposition}


\begin{proposition}
Given domain of interest $D_{n,m} = (U,W,\omega_U,\omega_W) = D$ and $(i,j) \in [1..(n-\omega_U+1)] \times [1..(m-\omega_W+1)]$, if $\min_{i',j'} DTWMU_D[i',j'] < DTWML_D[i,j]$, then $(i,j) \neq argmin_{i,j} DTWM_D[i,j]$
\label{prop:prune}
\end{proposition}




\begin{figure}
    \centering
    \includegraphics[width=0.9\linewidth]{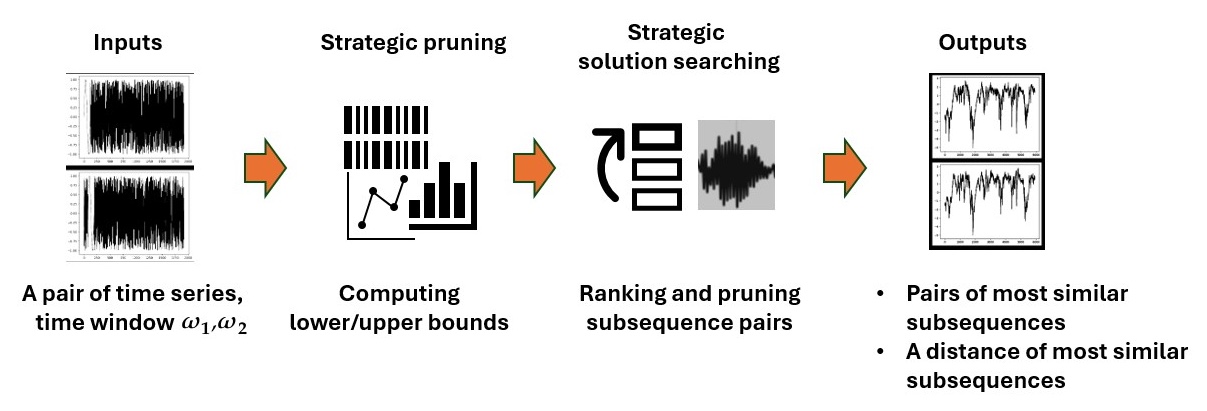}
    \caption{A high-level overview of the proposed framework. Given a pair of time series and time window $\omega_1,\omega_2$, the framework infers the most similar subsequences of length $\omega_1,\omega_2$ respectively.}
    \label{fig:frameworkOverview}
\end{figure}

\setlength{\intextsep}{0pt}
\IncMargin{1em}
\begin{algorithm2e*}
\caption{VLSubsequenceInferFunction}
\label{algo:VLfollowingMotifsInferFunction}
\SetKwInOut{Input}{input}\SetKwInOut{Output}{output}
\Input{Time series $U = (u_1,\dots,u_n)$, $W = (w_1,\dots,w_m)$, time windows $\omega_U$, $\omega_W$, and a distance function $dist(\cdot,\cdot)$. }
\Output{$SolutionIdx$, a set of $(a,b)$ where subsequence $U_{a,\omega_U}$ of length $\omega_U$ is most similar to subsequence $W_{b,\omega_W}$ of length $\omega_W$ and $shortestDist$, the distance between these subsequences.}
\begin{small}
\SetAlgoLined
\uIf{$\omega_W > \omega_U$}
{
    \nl     Swap $U$ and $W$,  $\omega_U$ and $\omega_W$\;
}
\nl Create $n\times m$ array M where $M[i,j]=dist(u_i,w_j)$\; 
\nl Create $n \times (m-\omega_W+1)$ array $MinPool$ where $MinPool[i,j] = min\{M[i;j,\dots,j+\omega_W-1]\}$\;
\nl Create $(n-\omega_U+1) \times (m-\omega_W+1)$ array $MinPath$ where $MinPath[i,j] = sum\{MinPool[i,\dots,i+\omega_U-1;j]\}$\;
\nl Create $(n-\omega_U+1) \times (m-\omega_W+1)$ array $MaxPath$ where $MaxPath[i,j] = \sum_{k=0}^{\omega_W-2} M[i+k,j+k] + \sum_{k=\omega_W-1}^{\omega_U-1} M[i+k,j+\omega_W-1]$\;
\tcc{Find the minimum element in $MaxPath$ matrix.}
\nl   $MinOfMaxPath = \min\{MaxPath\}$\;
\tcc{Strategically prune subsequence pairs using upper/lower bounds.}
\nl   $UnsortedCandidateSolutions = \{(i,j) \mid MinPath[i,j] \leq MinOfMaxPath\}$\;
\nl   $CandidateSolutions = [(i_k,j_k) \in UnsortedCandidateSolutions]$ array where $MinPath[i_k,j_k] \leq MinPath[i_{k+1},j_{k+1}]$\;
\tcc{Strategically search for optimal subsequence pairs by ranking and pruning pairs.}
\nl   $(SolutionIdx,shortestDist) = FindOptimalSolutions(M,\omega_U,\omega_W,CandidateSolutions,MinPath)$\;
\nl Return $(SolutionIdx,shortestDist)$\;
\end{small}
\end{algorithm2e*}\DecMargin{1em}

\section{Methods}
\label{sec:method}
To solve \VLSSIP, given two multidimensional time series and time window of subsequences $\omega_1,\omega_2$ as inputs, Algorithm~\ref{algo:VLfollowingMotifsInferFunction} can be used to find the most similar pairs of subsequences with length $\omega_1$ and $\omega_2$. Figure~\ref{fig:frameworkOverview} shows the overview of  Algorithm~\ref{algo:VLfollowingMotifsInferFunction}. After getting inputs, the algorithm computes the euclidean distances between each time step of both time series; then, it computes lower- and upper- bound distances of each subsequence pair. Afterward, 1) it prunes any subsequence pairs whose lower bounds are greater than an upper bound of any other pairs since their distances cannot be the minimum. Next, 2) the algorithm strategically ranks and searches subsequence pairs by computing an exact DTW distance of the subsequence pairs with the most potential to have the smallest distance. Then, the algorithm prunes any other subsequence pairs whose lower bounds are greater than this exact DTW distance. The process continues until there is no more subsequence pairs to search for. The algorithm finally reports the most similar pairs of subsequences with length $\omega_1$ and $\omega_2$.  Below is the detail of Algorithm~\ref{algo:VLfollowingMotifsInferFunction}. We also show that Algorithm~\ref{algo:VLfollowingMotifsInferFunction} always provides the solution for the Problem~\ref{prob:VLSSIP} in Theorem~\ref{thm:proofofcorrectness}.

\setlength{\intextsep}{0pt}
\IncMargin{1em}
\begin{algorithm2e*}
\caption{FindOptimalSolutions}
\label{algo:FindOptimalSolutions}
\SetKwInOut{Input}{input}\SetKwInOut{Output}{output}
\Input{Distance matrix $M$, time window $\omega_U$ and $\omega_W$, array of starting indices of subsequence pairs  $CandidateSolutions$, matrix of minimum paths $MinPath$.}
\Output{$SolutionIdx$, a set of $(a,b)$ where subsequence $U_{a,\omega_U}$ of length $\omega_U$ is most similar to $W_{b,\omega_W}$ of length $\omega_W$ and $shortestDist$, the distance between these subsequences.}
\begin{small}
\SetAlgoLined
\nl Create $SolutionIdx=\emptyset$\;
\nl $shortestDist = Inf$\;

\tcc{Find the DTW distance of a subsequence pair that is not yet pruned.}
\nl \For{$i=1$ to $|CandidateSolutions|$}
{
        \nl     $(a,b) = CandidateSolutions[i]$\;
        \tcc{Skip the pair if its lower-bound distance is still greater than the current solution}
        \If{$shortestDist < MinPath[a,b]$}
        {
            \nl     Break\;
        }
\nl    
$shortestDist_{new}=DynamicTimeWarping(M,\omega_U,\omega_W,(a,b))$\;
\tcc{Update the shortest distance solution if we find a better one.}
        \uIf{$shortestDist_{new} < shortestDist$}
        {
            \nl     $shortestDist = shortestDist_{new}$\;
            \nl    $SolutionIdx.clear()$\;
            \nl    $SolutionIdx.add((a,b))$\;
        }
        \uElseIf{$shortestDist_{new} = shortestDist$}
        {
            \nl    $SolutionIdx.add((a,b))$\;
        }
        \Else{
        \nl     Continue\;
        }
}
\nl    Return $(SolutionIdx,shortestDist)$\;
\end{small}
\end{algorithm2e*}\DecMargin{1em}

\begin{theorem}
\label{thm:proofofcorrectness}
Algorithm~\ref{algo:VLfollowingMotifsInferFunction} with time series $U=(u_1,\dots,u_n),W=(w_1,\dots,w_m)$ and time window $\omega_U,\omega_W$ as inputs provides an optimal solution for \VLSSIP.
\end{theorem}
\begin{proof}
    Algorithm ~\ref{algo:VLfollowingMotifsInferFunction} exhausts all possible cases of pairs of subsequences with length $\omega_U$ and $\omega_W$. However, we will show that when the algorithm prunes certain pairs of sequences; it only prunes those that cannot provide an optimal solution. 

(Forward direction)\\
    Our goal here is to find a set of pairs of subsequences with the length $\omega_U$ and $\omega_W$ that has the minimum DTW distance. In other words, the solution can be formalized as $Solutions = \{\{U_{i,\omega_U},W_{j,\omega_W}\} \mid (i,j) = argmin_{i,j}DTWM_D[i,j]\}$ where $D =$ domain of interest $D_{n,m} = (U,W,\omega_U,\omega_W)$
    
    Note that we can trivially construct $Solutions$ from $SolutionIdx$, a set of starting indices of subsequence pairs. In other words, $Solutions = \{\{U_{i,\omega_U},W_{j,\omega_W}\} \mid (i,j) \in SolutionIdx\}$ where $SolutionIdx = \{(i,j) \mid (i,j) = argmin_{i,j}DTWM_D[i,j]\}$.

    We will prove that Algorithm ~\ref{algo:VLfollowingMotifsInferFunction} outputs $SolutionIdx$.

    On line 4, we construct $MinPath = DTWML_D$.
    On line 5, we construct $MaxPath = DTWMU_D$.
    On line 6, we compute $MinOfMaxPath = \min_{i,j} DTWMU_D[i,j]$.

    It follows from Prop. ~\ref{prop:prune} that $\{(i,j) \mid MinOfMaxPath < MinPath[i,j]\} \cap SolutionIdx = \emptyset$

    Assume that we start with the set of all pairs of starting indices $AllIdx = \{(i,j) \mid \exists DTWM_D[i,j]\}$.

    On line 7, we then obtain $UnsortedCandidateSolutions$ from $AllIdx$ by eliminating index $(i,j) \in AllIdx$ if it doesn't belong in $SolutionIdx$. Formally, $SolutionIdx \subseteq UnsortedCandidateSolutions = AllIdx - \{(i,j) \mid \min_{i',j'} MaxPath[i',j'] < MinPath[i,j]\} = \{(i,j) \mid \min_{i',j'} MaxPath[i',j'] \geq MinPath[i,j]\} = \{(i,j) \mid MinOfMaxPath \geq MinPath[i,j]\}$
    
    On line 8, we obtain $CandidateSolutions$ by sorting $UnsortedCandidateSolutions$. This doesn't dismiss any valid solutions.

    On line 9, $FindOptimalSolutions()$ iterates through $(i,j) \in CandidateSolutions$ and puts $(i^*,j^*) \in CandidateSolutions$ in $SolutionIdx$ if $DTWM_D[i^*,j^*] = shortestDist \leq DTWM_D[i',j']$ for all $i',j'$ and ignore the rest of $(i,j)$ if $shortestDist < MinPath[i,j]$.

    We then return $SolutionIdx$ and $shortestDist$ on line 10.

(Backward direction)\\ 
    We will prove that Algorithm~\ref{algo:VLfollowingMotifsInferFunction} provides a complete set of solutions by contradiction. 
    
    Assume that there exists $(i^*,j^*) \notin  SolutionIdx$ such that $(i^*,j^*) = argmin_{i,j}DTWM_D[i,j]$. Let $(i',j') = argmin_{i',j'} DTWMU_D[i',j']$.

    It follows from the assumption that $MinPath[i^*,j^*] = DTWML_D[i^*,j^*] \leq DTWM_D[i^*,j^*] \leq DTWM_D[i',j'] \leq DTWMU_D[i',j'] = MinOfMaxPath$

    Since $MinPath[i^*,j^*] \leq MinOfMaxPath$, $(i^*,j^*) \in UnsortedCandidateSolutions$. After sorting, $(i^*,j^*) \in CandidateSolutions$

    On line 9, it is the case that $(i^*,j^*) \in SolutionIdx$ since $DTWM_D[i^*,j^*] \leq DTWM_D[i',j']$ for all $i',j'$ (from assumption). This contradicts the assumption, and, therefore, Algorithm~\ref{algo:VLfollowingMotifsInferFunction} must also output $(i^*,j^*)$.


 Therefore, Algorithm~\ref{algo:VLfollowingMotifsInferFunction} provides an optimal solution for the Problem~\ref{prob:VLSSIP}.
\end{proof}


\subsection{Time Complexity}
In the best case, Algorithm~\ref{algo:VLfollowingMotifsInferFunction} performs within $O(n\times m)$ time steps where $n,m$ are the length of time series since it just computes ordinary DTW. For the average case, given $1 < k < n\times m$, the algorithm performs within $O(k\times n\times m)$ time steps since the algorithm prunes certain subsequence pairs and leaves only $k$ pairs to compute DTW (DTW uses $\omega_U\times \omega_W$ time steps). For real-world datasets, the similar subsequences generally have lower distance than random matching of subsequences. This makes our algorithm be able to prune most of subsequence pairs out. Hence, $k$ is typically much lower than $n\times m$. For the worst case scenario, the algorithm performs within $O(n^2m^2)$, the same as using $DTW$ to compute all subsequence pairs. 

\section{Experiments}


\subsection{Experimental setup}

To determine performance of methods, there are two aspects we considered in this work: running time and correctness of inference. For the running time, we vary the length of time series and $\omega_1,\omega_2$ and compare running time for each method. For the correctness of inference, given a ground truth of intervals (GT) within each time series that contains the most similar subsequence (MSS), we measure the performance of methods as follows. The number of true positive cases (TP) represents intersection between the predicted and GT of MSS. The false positive (FP) occurs when a method infers that there is a position $i$ in MSS but the GT disagrees. The false negative (FN) occurs when  MSS contains position $i$ w.r.t. GT but the method fails to infer $i$ as a position in MSS.  The precision, recall, F1 score can be computed from these variables. We vary the level of noise within the simulation datasets to evaluate the robustness of methods. Suppose $W$ is a time series generated from the model, the uniform noise time series $\mathcal{U}$ is added to $U$ with the following equation. 
\begin{equation}
    \hat{W}=(1-\gamma)\times W + \gamma\times \mathcal{U} 
\end{equation}

We vary $\gamma \in {0.1,0.2,0.3,0.4,0.5}$ for our analysis. All experiments were perform on a PC with OS: Ubuntu 22.04.4 LTS x86\_64, Kernel: 6.8.0-59-generic, CPU: AMD Ryzen 7 5800X (16) @ 3.800GHz, Memory: 32033MiB.

\subsection{Time series simulation}

\begin{figure}
    \centering
    \includegraphics[width=0.5\linewidth]{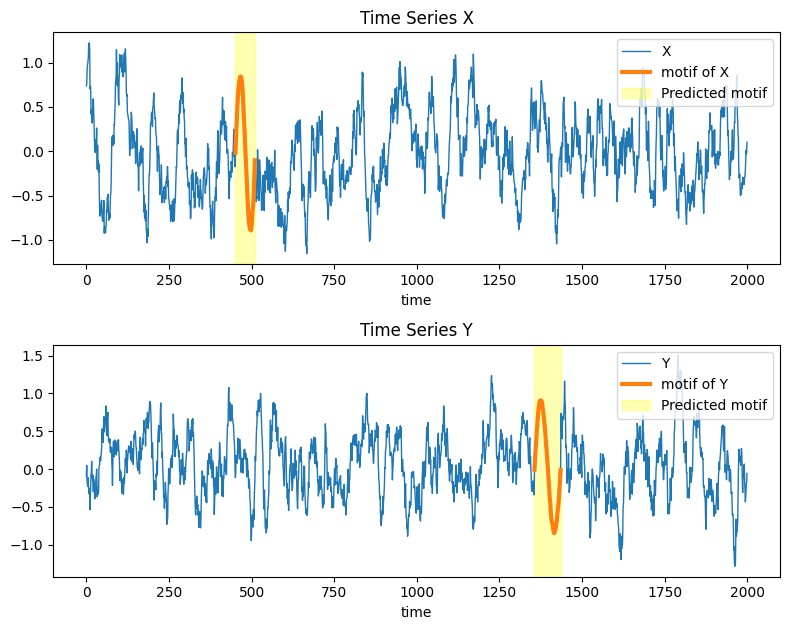}
    \caption{An example of simulated time series that contains the most similar subsequences with different lengths: 60 for X and 80 for Y (orange). The proposed algorithm correctly identified intervals that contain  most similar subsequence in both time series (yellow highlight). }
    \label{fig:exSimTS}
\end{figure}

We use the moving average model with time delay 20 steps and normal $\mathcal{N}(0,4)$ where $\mathcal{N}(\mu,\sigma^2)$ is a normal distribution with mean $\mu$ and variance $\sigma^2$. The generated motifs are sine signals with length $l$ whose frequencies $=1/l$ Hz and amplitude $= 1$. If applicable, we will use time window = 60 and 80. An example of simulated time series is in Fig.~\ref{fig:exSimTS}. The proposed algorithm correctly identified intervals that contain the most similar subsequences with different lengths in both time series.

\subsection{Real world dataset}



\subsubsection{Baboon leader/follower time series}

The time series of trajectories of this baboons' coordinated movement were from GPS collars of an olive baboon (\emph{Papio anubis}) troop in the wild in Mpala Research Centre, Kenya~\cite{strandburg2015shared}. After filtering, it consisted of 16 baboons, whose leader in this particular event had ID = 3. The coordination event began on Aug 2nd, 2012, at 6:00 AM and ended 10 minutes later~\cite{amornbunchornvej2018coordination}. The ID3 initiated the coordinated movement and everyone followed it for 100 seconds; then, ID1 led the group~\cite{FLICAtkdd}. The Variable-Lag-Granger causal relation between time series of ID3 directions as a cause and the aggregate time series of rest of the group can be detected in this event~\cite{amornbunchornvej2021variable}. In this work, we analyze the normalized two-dimensional time series of positions of 16 baboons with the length 600 time steps (one time step per second). 

\subsubsection{Stock prices of companies in similar sector vs those in unrelated sector}

The time series of Nvidia Corp. (Nvidia), Vishay Intertechnology, Inc. (Vishay), and Tyson Foods, Inc. (Tyson) 2020 stock prices were downloaded from Yahoo Finance. In this work, we find the most similar subsequences (time window = 90 days) of the normalized time series of Nvidia, Vishay, and Tyson stock prices, each of which has 253 time steps (one time step per trading day).
\section{Results}

\subsection{Simulation results}
There are four methods that we performed analysis. The brute force method that performs sliding window to select all possible subsequence pair w.r.t. given $\omega_U,\omega_W$, then computes DTW for each pair. The Sakoe Chiba is the brute force method with Sakoe-Chiba constraint. The Strategic pruning (SP) is our approach and the last one is the SP with Sakoe-Chiba constraint.
\begin{figure}
    \centering
    \includegraphics[width=0.7\linewidth]{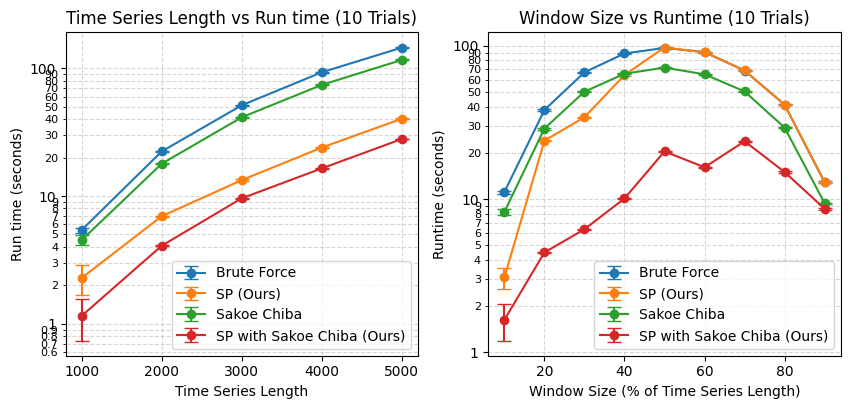}
    \caption{A comparison of running time of methods vs. a) the length of time series to find the most similar subsequences and b) the window size with the fix time series length as 2000.}
    \label{fig:res-RuntimeVsLength}
\end{figure}
Fig.~\ref{fig:res-RuntimeVsLength} a) shows the running time of methods with different lengths of time series. Our proposed method consistently used less time than others; it used around one-five of time compared against the brute-force (SP with Sakoe-Chiba constraint). The time growing rates for all methods are slightly non-linear.  For Fig.~\ref{fig:res-RuntimeVsLength} b), it shows that when setting the windows $\omega_U=\omega_W$ and vary the windows, our SP with Sakoe-Chiba constraint performed fastest. The SP performed well only when the window size is below 40\% of time series length.  

\begin{figure}
    \centering
    \includegraphics[width=0.7\linewidth]{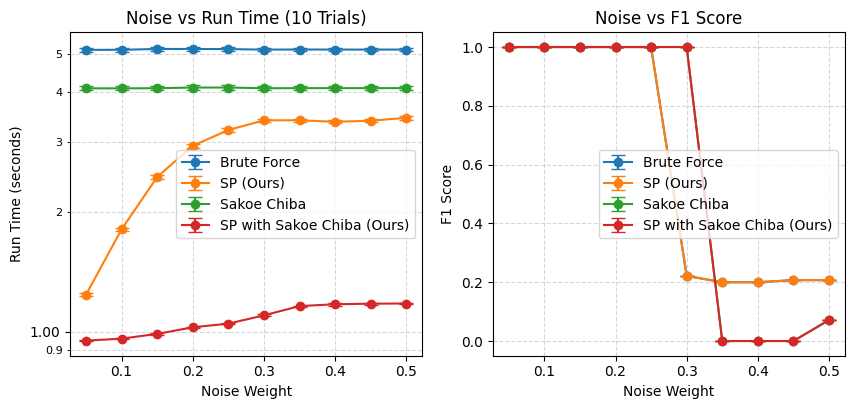}
    \caption{A comparison of noise level vs. a) the running time and b) F1 score of inferring correct subsequences. The time series length is 2000 time steps.}
    \label{fig:res-Sensitivity}
\end{figure}

For sensitivity analysis, Fig.~\ref{fig:res-Sensitivity} a) shows the running time of methods vs. noise. When the noise level is lower, our sliding window approach used one fifth of brute-force time. However, when the noise level increased, all methods used more time to compute. 

Fig.~\ref{fig:res-Sensitivity} b) shows that the performance of inferring the most similar subsequences of all methods dropped after the noise level reach beyond 0.25. Briefly, all methods can tolerate some level of noises but too much noise make it is impossible to infer correct subsequences.

\subsection{Case study: leadership of coordinated movement of baboons}
\label{subsec:caseBaboon}

\begin{figure}
    \centering
    \includegraphics[width=.8\linewidth]{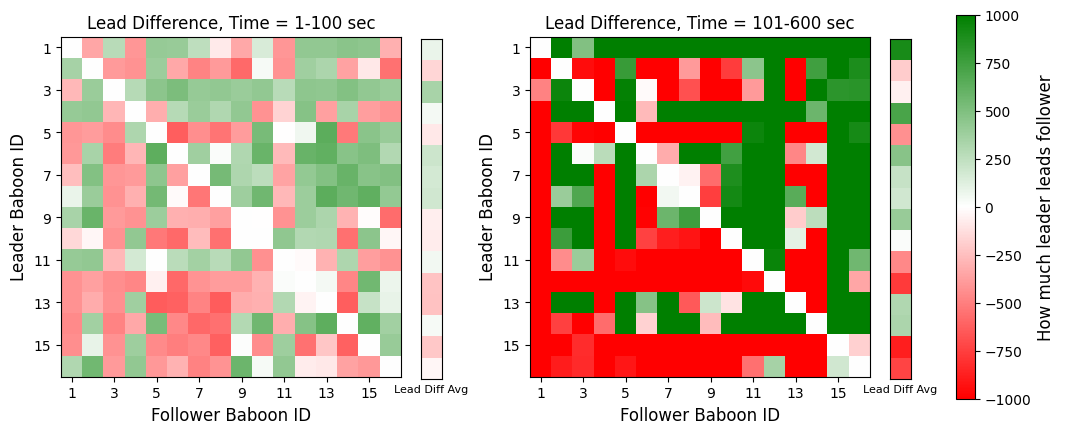}
    \caption{Heatmaps of lead difference during a) the first 100 seconds and b) the last 500 seconds. The y-axis is leader and x-axis is follower. Given top-1000 most similar subsequences, how much leader leads follower = \# of leading subsequence pairs - \# of following subsequence pairs (defined in Subsection~\ref{subsec:caseBaboon}).}
    \label{fig:res-baboon}
\end{figure}

Fig.~\ref{fig:res-baboon} shows heatmaps of lead difference. Let $idx\_l_i$ = starting index of subsequence from Leader ID $i$ and  $idx\_f_j$ = starting index of subsequence from Follower ID $j$. We define leading subsequence pair as a pair where $idx\_l_i < idx\_f_j$, and following subsequence pair as a pair where $idx\_l_i > idx\_f_j$. Given top-1000 most similar subsequences of Leader ID $i$ and Follower ID $j$, each cell in the heatmap has a value of $L-F$ where $L$ = number of leading subsequence pairs, and $F$ = number of following subsequence pairs. A cell is green when $L>F$ and red when $L<F$.


The window size is $\omega_U=\omega_W=60$.  In Fig.~\ref{fig:res-baboon} a), ID3 has the most of green cells during the first 100 seconds, which is consistent with the ground truth that ID3 initiated the group movement. 

In Fig.~\ref{fig:res-baboon} b), ID1 has the most of green cells during the last 500 seconds. This is also consistent with the ground truth that ID1 led the group movement most of the time. Moreover, in the baboon dataset, our approach ran up to 20 times faster than the brute-force approach. 

\subsection{Case study: Stock prices of companies in similar sector vs those in unrelated sector}

\begin{figure}
    \centering
    \includegraphics[width=1\linewidth]{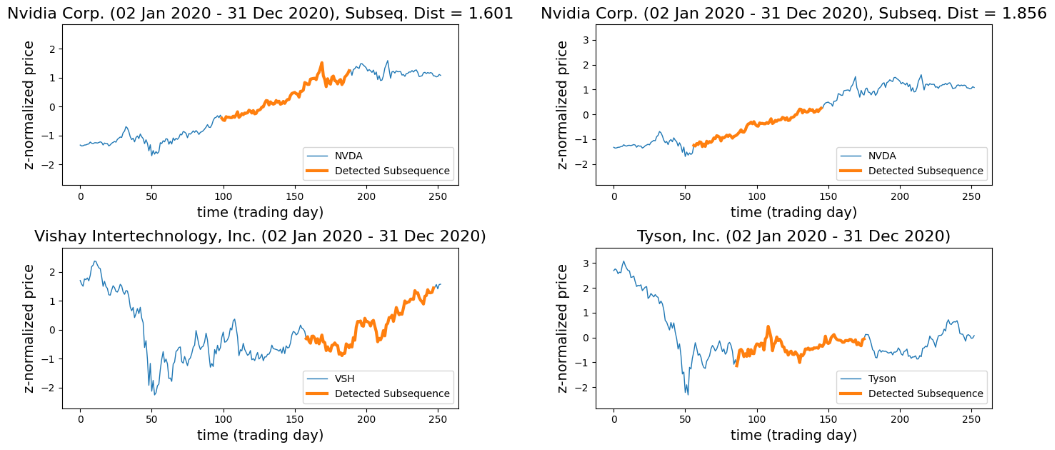}
    \caption{The most similar subsequences (orange) of time series of prices of NVIDIA stock prices vs. (left) common sector company (vishay) and (right) different sector company chain (Tyson).  }
    \label{fig:res-NVDA}
\end{figure}

We analyzed this dataset using $\omega_U=\omega_W=90$ or three months. Fig.~\ref{fig:res-NVDA} (left) shows the most similar subsequences (orange) of prices of NVIDIA stock vs. Vishay stock (company in related sector). The match subsequences established the same trend and pattern with minimum distance as 1.601, which is consistent with the fact that both companies operate in a similar sector and might be influenced by similar factors. In contrast, Fig.~\ref{fig:res-NVDA} (right) shows the most similar subsequences (orange) of prices of NVIDIA stock vs. Tyson stock (company in a different sector) with higher distance as 1.856. The matching in this case has different pattern, which shows that both companies are unrelated. This result demonstrates the utilization of using our method to find pattern of potential dependency in financial data.  
\section{Conclusion}
In this work, we proposed an algorithm providing the exact solution of finding the most similar multidimensional subsequences between time series where there is a difference in length both between time series and between subsequences. The algorithm was built based on theoretical guarantee of correctness and efficiency. The result in simulation datasets illustrated that our approach not just only provided correct solution but also utilized running time only quarter of time compared against the baseline approaches. In real-world datasets, it extracted the most similar subsequences even faster than simulation and provided insights regarding the situation in stock market and following relations of multidimensional time series of baboon movement. The limitation of proposed method is that the meaningful subsequences must have lower distance than any random matching of other subsequences. The future work is needed to improve this aspect.  Our approach can be used for any time series. The code and datasets of this work are provided for the public use at this \href{https://anonymous.4open.science/r/dtw-pair-0D9E/README.md}{link}.


\medskip

{
\small
\bibliographystyle{apalike}



}


\appendix

\section*{Appendix: Details of proofs and algorithms}
\label{sec:apxProof}
The details of proofs of proposition and lemma provided in Section~\ref{sec:probstat}.

\begin{proposition}
(Proposition~\ref{prop:lowerbound})
Given subsequence $U_{i,\omega_U}$, $W_{j,\omega_W}$ where $\omega_U \geq \omega_W$, and a distance matrix $M_{U, W} = M$, the following inequality holds:

$DTW_L(U_{i,\omega_U}, W_{j,\omega_W}) \leq DTW(U_{i,\omega_U}, W_{j,\omega_W})$
\end{proposition}
\begin{proof}
    Assuming that the warping path $P^*$ is used to compute $DTW(U_{i,\omega_U}, W_{j,\omega_W})$, each element in $P^*$ has to have its first coordinate go through each integer in $[i..(i+\omega_U-1)]$ since the first and last elements in $P^*$ are $(i,j)$ and $(i+\omega_U-1,j+\omega_W-1)$, respectively, and each successive element in $P^*$ can change its first coordinate by only 0 or +1. Formally, $P^*$ is a sequence formed by consecutively concatenating the sequence $P_i,P_{i+1},\dots,P_{i+\omega_U-1}$ given that sequence $P_{k \in [i..(i+\omega_U-1)]} = (p_1,p_2,\dots,p_{l_k})$ whose element $p_{n \in [1..l_k]} = (k,q_{k,n})$ and $q_{k,n} \in [j..(j+\omega_W-1)]$.

    $DTW(U_{i,\omega_U}, W_{j,\omega_W}) = DistSub(U_{i,\omega_U}, W_{j,\omega_W},P^*) = \sum_{k=i}^{i+\omega_U-1} \sum_{n=1}^{l_k} M[k,q_{k,n}] \geq \sum_{k=i}^{i+\omega_U-1} \min \{M[k;j,\dots,j+\omega_W-1]\} = DTW_L(U_{i,\omega_U}, W_{j,\omega_W})$
\end{proof}

\begin{lemma}
(Lemma~\ref{lemma:validpath})
Given subsequence $U_{i,\omega_U}$, $W_{j,\omega_W}$ where $\omega_U \geq \omega_W$, if sequence $P = (p_n)_{n \in [0..(\omega_U-1)]}$ where $p_n = (i+n,j+n)$ if $n < \omega_W-1$ and $(i+n,j+\omega_W-1)$ if $n \geq \omega_W-1$, then $P$ is a warping path between $U_{i,\omega_U}$ and $W_{j,\omega_W}$
\end{lemma}
\begin{proof}
$P = (p_0,p_2,\dots,p_{\omega_U-1})$ where $p_0 = (i,j)$ and $p_{\omega_U-1} = (i+\omega_U-1,j+\omega_W-1)$ so $P$ satisfies (1) the boundary condition.

For $p_{k \in [0..(\omega_U-1)]}$, $p_{k+1} - p_k = (1,1)$ when $k < \omega_W-1$ and $p_{k+1} - p_k = (1,0)$ when $k \geq \omega_W-1$ so $P$ satisfies (2) the continuity and monotonicity condition.

Since $P$ satisfies (1) and (2), $P$ is a warping path between $U_{i,\omega_U}$ and $W_{j,\omega_W}$
\end{proof}

\begin{proposition}
(Proposition~\ref{prop:upperbound})
Given subsequence $U_{i,\omega_U}$, $W_{j,\omega_W}$ where $\omega_U \geq \omega_W$, and a distance matrix $M_{U, W} = M$, the following inequality holds:

$DTW(U_{i,\omega_U}, W_{j,\omega_W}) \leq DTW_{up}(U_{i,\omega_U}, W_{j,\omega_W})$

\end{proposition}
\begin{proof}
From Lemma ~\ref{lemma:validpath}, let a sequence $P = (p_n)_{n \in [0..(\omega_U-1)]}$ where $p_n = (i+n,j+n)$ if $n < \omega_W-1$ and $(i+n,j+\omega_W-1)$ if $n \geq \omega_W-1$ be a warping path between $U_{i,\omega_U}$ and $W_{j,\omega_W}$. 

$DTW(U_{i,\omega_U}, W_{j,\omega_W}) = \min_{P'} DistSub(U_{i,\omega_U}, W_{j,\omega_W},P') \leq DistSub(U_{i,\omega_U}, W_{j,\omega_W},P) = \sum_{k=0}^{\omega_W-2} dist(u_{i+k},w_{j+k}) + \sum_{k=\omega_W-1}^{\omega_U-1} dist(u_{i+k},w_{j+\omega_W-1}) = \sum_{k=0}^{\omega_W-2} M[i+k,j+k] + \sum_{k=\omega_W-1}^{\omega_U-1} M[i+k,j+\omega_W-1] = DTW_{up}(U_{i,\omega_U}, W_{j,\omega_W})$
\end{proof}

\begin{proposition}
(Proposition~\ref{prop:prune})
Given domain of interest $D_{n,m} = (U,W,\omega_U,\omega_W) = D$ and $(i,j) \in [1..(n-\omega_U+1)] \times [1..(m-\omega_W+1)]$, if $\min_{i',j'} DTWMU_D[i',j'] < DTWML_D[i,j]$, then $(i,j) \neq argmin_{i,j} DTWM_D[i,j]$
\end{proposition}
\begin{proof}
Let $(i',j') = argmin_{i',j'} DTWMU_D[i',j']$ and $(i,j) \in [1..(n-\omega_U+1)] \times [1..(m-\omega_W+1)]$

From Def. ~\ref{def:DTWM} and Prop. ~\ref{prop:upperbound}, $DTWM_D[i',j'] =DTW(U_{i',\omega_U}, W_{j',\omega_W}) \leq  DTW_{up}(U_{i',\omega_U}, W_{j',\omega_W}) = DTWMU_D[i',j'] = \min_{i',j'} DTWMU_D[i',j'])$ 

From Def. ~\ref{def:DTWM} and Prop. ~\ref{prop:lowerbound}, $DTWML_D[i,j] = DTW_L(U_{i,\omega_U}, W_{j,\omega_W}) \leq DTW(U_{i,\omega_U}, W_{j,\omega_W}) = DTWM_D[i,j]$

From the premise, $DTWM_D[i',j'] \leq \min_{i',j'} DTWMU_D[i',j'] < DTWML_D[i,j] \leq DTWM_D[i,j]$

Since $DTWM_D[i',j'] < DTWM_D[i,j]$, $(i,j) \neq argmin_{i,j} DTWM_D[i,j]$
\end{proof}

The detail of dynamic time warping algorithm is given below.

\setlength{\intextsep}{0pt}
\IncMargin{1em}
\begin{algorithm2e*}
\caption{DynamicTimeWarping}
\label{algo:DTW}
\SetKwInOut{Input}{input}\SetKwInOut{Output}{output}
\Input{Distance matrix $M$, time windows $\omega_U$ and $\omega_W$, starting indices of a subsequence pair $(a,b)$. }
\Output{$shortestDist$, DTW distance between subsequences $U_{a,\omega_U}$ and $W_{b,\omega_W}$. }
\begin{small}
\SetAlgoLined
\tcc{$AD$ is an accumulated distance matrix.}
\nl     Create $\omega_U \times \omega_W$ array $AD$\;
\nl     $AD[i,1]=\sum_{k=0}^{i-1} M[a+k,b]$ for $i \in [1..\omega_U]$\;
\nl     $AD[1,j]=\sum_{k=0}^{j-1} M[a,b+k]$ for $j \in [1..\omega_W]$\;
\nl     $AD[n,m]=min(AD[n-1,m],AD[n,m-1],AD[n-1,m-1])+M[a+n-1,b+m-1]$ for $n \in [2..\omega_U]$ and $m \in [2..\omega_W]$\;
\nl     $shortestDist = AD[\omega_U,\omega_W]$\; 
\nl     Return $shortestDist$\;
\end{small}
\end{algorithm2e*}\DecMargin{1em}

\section*{Appendix: Results}

\begin{figure}
    \centering
    \includegraphics[width=1\linewidth]{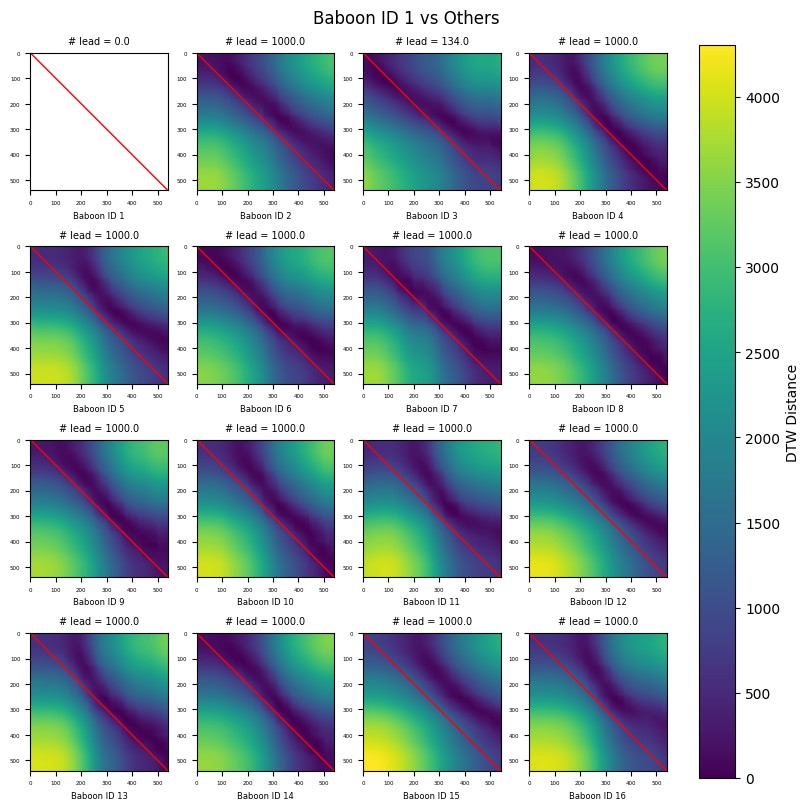}
    \caption{Heatmap of distances of ID1 as a leader vs. the rest. The vertical axis is leader and horizontal axis is a follower. The \#lead is a number of top 1000th most similar subsequences that the leader leads its follower.}
    \label{fig:ID1baboon}
\end{figure}

\begin{figure}
    \centering
    \includegraphics[width=1\linewidth]{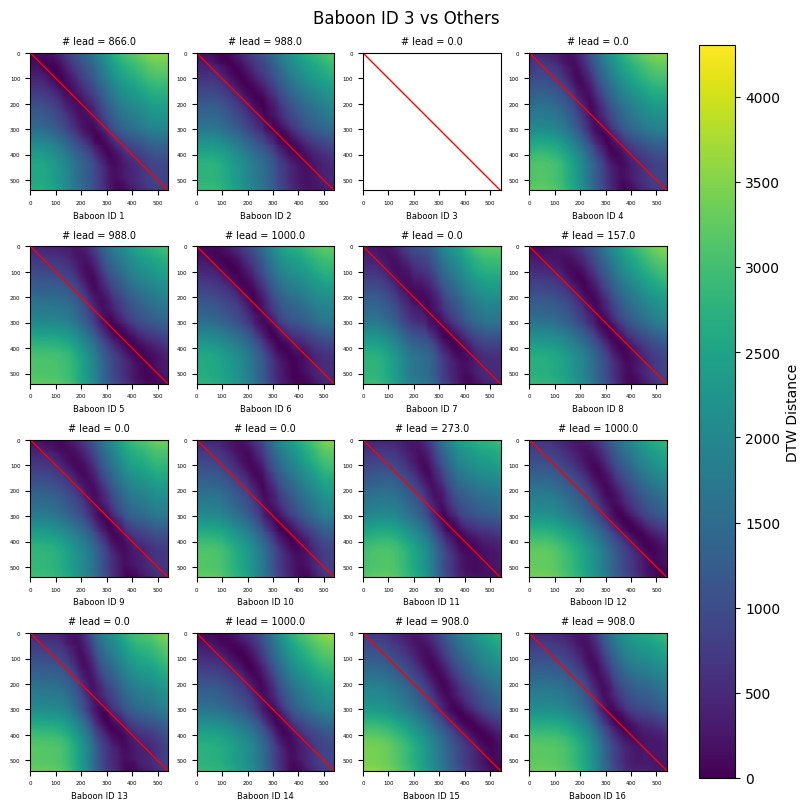}
    \caption{Heatmap of distances of ID3 as a leader vs. the rest. The vertical axis is leader and horizontal axis is a follower. The \#lead is a number of top 1000th most similar subsequences that the leader leads its follower.}
    \label{fig:ID3baboon}
\end{figure}


\end{document}